%% file: arxiv_version.tex
\documentclass[11pt]{article}  

\input{packages_only_arxiv.tex}

\input{theorems_std.tex}
\input{macros.tex}
\newcommand{\citep}[1]{(\cite{#1})}

\begin{document}

\title{On the ability of neural nets to express distributions}

\author{Holden Lee\thanks{Princeton University, Mathematics Department}, 
Rong Ge\thanks{Duke University, Computer Science Department},
Tengyu Ma\thanks{Princeton Univerisity, Computer Science Department},
Andrej Risteski\thanks{Princeton Univerisity, Computer Science Department},
Sanjeev Arora\thanks{Princeton Univerisity, Computer Science Department. Supported by NSF grants CCF- 1302518, CCF-1527371,  Simons Investigator Award, Simons Collaboration Grant, and ONR- N00014-16-1-2329}}

\date{\today}
\maketitle
\begin{abstract}
Deep neural nets have caused a revolution in many classification tasks. A related ongoing revolution---also theoretically not understood---concerns their ability to serve as generative models for complicated types of data such as images and texts. These models are trained using ideas like variational autoencoders and Generative Adversarial Networks.

We take a first cut at explaining the expressivity of multilayer nets by giving a sufficient criterion for a function to be approximable by a neural network with $n$ hidden layers. A key ingredient is
Barron's Theorem \cite{Barron1993},
which gives a Fourier criterion for approximability of a function by a neural network with 1 hidden layer. We show that a composition of $n$ functions which satisfy certain Fourier conditions (``Barron functions'') can be approximated by a $n+1$-layer neural network.
 
For probability distributions, this translates into a criterion for a probability distribution to be approximable in Wasserstein distance---a natural metric on probability distributions---by a neural network applied to a fixed base distribution (e.g., multivariate gaussian).

Building up recent lower bound work, we also give an example function 
that shows that composition of Barron functions is more expressive than Barron functions alone.
\end{abstract}


\input{intro.tex}

\input{barron.tex}

\input{body.tex}

\input{separation.tex}

\section{Conclusion}

In this paper we show if a generative model can be expressed as the composition of $n$ Barron functions, then it can be approximated by a $n+1$-layer neural network. Along the way we proved a multi-layer version of the Barron's Theorem \cite{Barron1993}, and a key observation is to use Wasserstein distance $W^2$ as the distance measure between distributions. This partly explains the expressive power of neural networks as generative models. However, there are still many open problems: what natural transformations can be represented by a composition of Barron functions? Is there a separation between composition of $n$ Barron functions and composition of $n+1$ Barron functions? How can we learn such a representation efficiently? We hope this paper serves as a first step towards understanding the power of deep generative models.

\newpage
\printbibliography

\newpage

\appendix
\input{appendix.tex}
\input{bump_new.tex}

\input{separation-proof.tex}

\end{document}

%% file: packages_only_arxiv.tex
\usepackage{etex}

\usepackage{fullpage}

\usepackage{amsmath}
\usepackage{amssymb}
\usepackage{amsthm}
\usepackage{array}
\usepackage{bbm}
\usepackage{cancel}
\usepackage{cmap}
\usepackage{enumerate}
\usepackage{enumitem}
\usepackage{fancyhdr}
\usepackage{mathdots}
\usepackage{mathtools}
\usepackage{mathrsfs}
\usepackage{hyperref}
\usepackage{stackrel}
\usepackage{stmaryrd}
\usepackage{tabularx}
\usepackage{tikz}
\usepackage{ctable}
\usepackage{titlesec}
\usepackage{titletoc}
\usepackage{url}
\usepackage{verbatim}
\usepackage{wasysym}
\usepackage{wrapfig}
\usepackage{yhmath}
\usepackage[all,cmtip]{xy}

\usepackage{hyperref}
\usepackage[%
	firstinits=true,
	doi=false,
	url=false,
	isbn=false,
	backend=biber,
	style=alphabetic,hyperref]{biblatex}

\usetikzlibrary{calc,trees,positioning,arrows,chains,shapes.geometric,%
    decorations.pathreplacing,decorations.pathmorphing,shapes,%
    matrix,shapes.symbols,shadows,fadings}



%% file: theorems_std.tex
\newtheorem{thm}{Theorem}[section]
\newtheorem*{thm*}{Theorem}

\newtheorem*{prb*}{Problem}

\newtheorem*{ax*}{Axiom}

\newtheorem*{clm*}{Claim}

\newtheorem*{conj*}{Conjecture}
\newtheorem{cor}[thm]{Corollary}
\newtheorem{df}[thm]{Definition}
\newtheorem*{df*}{Definition}

\newtheorem*{ex*}{Example}

\newtheorem{lem}[thm]{Lemma}
\newtheorem*{lem*}{Lemma}

\newtheorem*{pos*}{Postulate}
\newtheorem{pr}[thm]{Proposition}
\newtheorem*{pr*}{Proposition}

\newtheorem*{qu*}{Question}

\newtheorem*{rem*}{Remark}

%% file: macros.tex
\def\shownotes{1}  \ifnum\shownotes=1
\newcommand{\authnote}[2]{$\ll$\textsf{\footnotesize #1 notes: #2}$\gg$}
\else
\newcommand{\authnote}[2]{}
\fi

\newcommand{\C}[0]{\mathbb{C}}

\newcommand{\E}[0]{\mathbb{E}}
\newcommand{\EE}[0]{\mathop{\mathbb E}}

\newcommand{\N}[0]{\mathbb{N}}

\newcommand{\R}[0]{\mathbb{R}}

\newcommand{\one}[0]{\mathbbm{1}}

\newcommand{\al}[0]{\alpha}
\newcommand{\be}[0]{\beta}
\newcommand{\ga}[0]{\gamma}
\newcommand{\Ga}[0]{\Gamma}
\newcommand{\de}[0]{\delta}
\newcommand{\De}[0]{\Delta}
\newcommand{\ep}[0]{\varepsilon}

\newcommand{\ph}[0]{\varphi}

\newcommand{\Te}[0]{\Theta}

\newcommand{\om}[0]{\omega}
\newcommand{\Om}[0]{\Omega}
\newcommand{\si}[0]{\sigma}

\newcommand{\nin}[0]{\not\in}

\newcommand{\sub}[0]{\subset}

\newcommand{\subeq}[0]{\subseteq}

\newcommand{\iy}[0]{\infty}

\newcommand{\rc}[1]{\frac{1}{#1}}
\newcommand{\prc}[1]{\pa{\rc{#1}}}

\newcommand{\fc}[2]{\frac{#1}{#2}}
\newcommand{\sfc}[2]{\sqrt{\frac{#1}{#2}}}
\newcommand{\pf}[2]{\pa{\frac{#1}{#2}}}

\newcommand{\dd}[2]{\frac{d #1}{d #2}}

\newcommand{\pdt}[2]{\frac{\partial^2 #1}{\partial {#2}^2}}

\newcommand{\nb}[0]{\nabla}

\newcommand{\dx}{\,dx}

\newcommand{\ab}[1]{\left| {#1} \right|}
\newcommand{\an}[1]{\left\langle {#1}\right\rangle}
\newcommand{\ba}[1]{\left[ {#1} \right]}

\newcommand{\ce}[1]{\left\lceil {#1}\right\rceil}

\newcommand{\pa}[1]{\left( {#1} \right)}

\newcommand{\ve}[1]{\left\Vert {#1}\right\Vert}

\newcommand{\ved}[0]{\ve{\cdot}}
\newcommand{\set}[2]{\left\{{#1}:{#2}\right\}}

\newcommand{\ol}[1]{\overline{#1}}

\newcommand{\ub}[2]{\underbrace{#1}_{#2}}

\newcommand{\wh}[1]{\widehat{#1}}

\newcommand{\diam}{\operatorname{diam}}

\newcommand{\Lip}[0]{\operatorname{Lip}}

\newcommand{\poly}{\operatorname{poly}}

\newcommand{\spn}{\operatorname{span}}

\newcommand{\Supp}{\operatorname{Supp}}

\newcommand{\Vol}[0]{\text{Vol}}

\providecommand{\cal}[1]{\mathcal{#1}}
\renewcommand{\cal}[1]{\mathcal{#1}}

\newcommand{\pull}[9]{
#1\ar@/_/[ddr]_{#2} \ar@{.>}[rd]^{#3} \ar@/^/[rrd]^{#4} & &\\
& #5\ar[r]^{#6}\ar[d]^{#8} &#7\ar[d]^{#9} \\}

\newcommand{\cmp}[9]{
\xymatrix{
#1 \ar[r]^{#4}{#5} \ar@/_2pc/[rr]^{#8}_{#9} & #2 \ar[r]^{#6}_{#7} & #3
}
}

\newcommand{\ha}[1]{\ar@{^(->}[#1]}
\newcommand{\ls}[1]{\ar@{-}[#1]}
\newcommand{\sj}[1]{\ar@{->>}[#1]}
\newcommand{\aq}[1]{\ar@{=}[#1]}
\newcommand{\acir}[1]{\ar@{}[#1]|-{\textstyle{\circlearrowright}}}
\newcommand{\acil}[1]{\ar@{}[#1]|-{\textstyle{\circlearrowleft}}}
\newcommand{\ard}[1]{\ar@{.>}[#1]}
\newcommand{\mt}[1]{\ar@{|->}[#1]}
\newcommand{\inm}[1]{\ar@{}[#1]|-{\in}}
\newcommand{\inr}{\ar@{}[d]|-{\rotatebox[origin=c]{-90}{$\in$}}}
\newcommand{\inl}{\ar@{}[u]|-{\rotatebox[origin=c]{90}{$\in$}}}

\newcommand{\sumr}[2]{\sum_{\scriptsize \begin{array}{c}{#1}\\{#2}\end{array}}}

\newcommand{\sumo}[2]{\sum_{#1=1}^{#2}}
\newcommand{\sumz}[2]{\sum_{#1=0}^{#2}}

\newcommand{\iiy}[0]{\int_0^{\infty}}
\newcommand{\iny}[0]{\int_{-\infty}^{\infty}}

\newcommand{\beq}[1]{\begin{equation}\llabel{#1}}
\newcommand{\eeq}[0]{\end{equation}}
\newcommand{\bal}[0]{\begin{align*}}
\newcommand{\eal}[0]{\end{align*}}
\newcommand{\ban}[0]{\begin{align}}
\newcommand{\ean}[0]{\end{align}}

\newcommand{\fixme}[1]{{\color{red}#1}}
\newcommand{\llabel}[1]{\label{#1}\text{\fixme{\tiny#1}}}

\newcommand{\arxiv}[1]{\url{http://www.arxiv.org/abs/#1}}

\newcommand{\vocab}[1]{\textbf{#1}} 

\allowdisplaybreaks[2]

\DeclareFontFamily{U}{wncy}{}
    \DeclareFontShape{U}{wncy}{m}{n}{<->wncyr10}{}
    \DeclareSymbolFont{mcy}{U}{wncy}{m}{n}
    \DeclareMathSymbol{\Sh}{\mathord}{mcy}{"58} 

%% file: intro.tex
\section{Introduction}

Deep neural networks 
have led to state-of-the-art performance on classification tasks in many domains such as computer vision, speech recognition, and reinforcement learning \citep{deepsurvey1,deepsurvey2}. 
One can view a neural network as a way to learn a function mapping inputs $x$ to outputs $y$. For image classification, the input is a  vector representing an image and
the output can be probabilities of being in various classes. 

But another recent (and less understood) use of neural networks is as generative models for complicated probability distributions, such as distributions over images on ImageNet, handwritten characters from various alphabets, or speech. 
Here the network may map a stochastic input---such as a uniform normal gaussian---to a realistic image. Such networks are trained using various methods such as variational autoencoders (\cite{kingma2013auto}, \cite{rezende2014stochastic}) or
generative adversarial networks (GANs) (\cite{goodfellow2014generative}). A GAN consists of a repeated zero-sum game between two networks: the \emph{generator} attempts to imitate a given probability distribution; it obtains its samples by passing a base distribution (e.g. a gaussian) through its neural network. The \emph{discriminator} attempts to distinguish between samples from the generator and the true distribution, and thus forces the generator to improve over many repetitions.

The current paper is concerned with the following natural question that appears not to have been studied before: Why are deep neural networks so well-suited to efficiently generate many distributions that occur in nature?
 
\subsection{Our work}

We give a sufficient criterion for a function to be approximable by a neural network with $n$ hidden layers (Theorem \ref{cor:multi}). This criterion holds with respect to any distribution of inputs supported on a compact set. 
As a consequence of our main result, we obtain a criterion for a distribution to be approximately generated by a neural network with $n$ hidden layers in the Wasserstein metric $W_2$, a natural metric on the space of distributions (Corollary~\ref{cor:dist}).

Our criterion relies on Fourier properties of the function.
We build on Barron's Theorem \cite{Barron1993},  
which says that if a certain quantity involving the Fourier transform is small, then the function can be approximated by a neural network with one hidden layer and a small number of nodes. 
Calling such a function a Barron function, 
our criterion roughly says that if a distribution is generated by a composition of $n$ Barron functions, then the distribution can be approximately generated by a neural network with $n$ hidden layers.

Many nice functions, such as polynomials and ridge functions, are Barron; this property is also preserved under natural operations such as linear combinations. Thus, our result says that if nature creates a distribution by starting from a base distribution (such as a gaussian) and applying a sequence of functions in this class, then we can also generate that distribution with a neural network.

This ``correspondence'' between compositions of Barron functions and multi-layer neural networks raises
questions analogous to those raised about neural nets: for example, are compositions of $k$
Barron functions more expressive than Barron functions?  Using a technique to lower-bound the Barron constant (Theorem~\ref{thm:barron-lb}), we show a separation theorem between Barron functions and composition of Barron functions (Theorem \ref{thm:separation}). This  parallels ---and is inspired by---the separation between 2-layer and 3-layer neural networks in \cite{eldan2015power}. 

\subsection{Related work}

Despite the practical success of neural networks, we lack a good theoretical understanding of their effectiveness.
An initial attempt to understand the effectiveness of neural networks was by their function approximation properties. A series of works showed that any continuous function in a bounded domain can be approximated by a sufficiently large 2-layer neural network (\cite{cybenko1989approximation}, \cite{funahashi1989approximate},  \cite{hornik1989multilayer}). However, the network size can be exponential in the dimension. Barron (\cite{Barron1993}) gave a upper bound for the size of the network required in terms of a Fourier criterion.
He showed that a function $f$ can be approximated in $L^2$ up to error $\ep$ by a 2-layer neural network with $O\pf{C_f^2}{\ep}$ units, where $C_f$ depends on Fourier properties of $f$. One remarkable consequence is that representationally speaking, neural nets can evade the curse of dimensionality: the number of parameters required to obtain a fixed error increases linearly, rather than superlinearly, in the number of dimensions. (Fixing the number of nodes in the hidden layer, the number of parameters scales linearly in the number of dimensions.)

However, such approximability results only explain a small part of the success of neural networks. 
Firstly, they only deal with 2-layer neural networks. Empirically speaking, deep neural networks---networks with many layers---appear to be much more effective than shallow neural networks. There have been several attempts to explain the effectiveness of deep neural networks. Following the paradigm in circuit complexity, one produces a function $f$ that can be computed by a deep neural network but requires exponentially many nodes to be computed by a shallow neural network. Eldan and Shamir (\cite{eldan2015power}) show a certain radial function can be approximated by a 3-layer neural net but not by a 2-layer neural net with a subexponential number of nodes. \cite{daniely2017depth} shows such a separation but with respect to the uniform distribution on the sphere.
Telgarsky (\cite{telgarsky2016benefits}) shows such a separation between $k^2$-layer and $k$-layer neural networks. Cohen, Sharir, and Shashua (\cite{cohen2015expressive}) show a separation for a different model, a certain type of convolutional neural net architecture. Kane and Williams (\cite{kane2016super}) show super-linear gate and super-quadratic wire lower bounds for depth-two and depth-three threshold circuits, which can be thought of as a boolean analogue to neural networks.

Secondly, these works---as well as our paper---do not address how to learn neural networks, or why the established method, gradient descent, has been so successful. \cite{Barron1993} and \cite{Barron1994} address the generalization theory, and show that the nodes can be chosen ``greedily''; however the optimization problem is nonconvex. Under the assumption that certain properties of the input distribution (related to the score function) are known and that the function is exactly representable by a 2-layer neural network, Janzamin, Sedghi, and Anandkumar (\cite{janzamin2015beating}) give an algorithm inspired by Barron's Fourier criterion and utilizing tensor decomposition, to learn 2-layer neural networks.

Finally, we note that the learnability for distributions has been studied for discrete distributions~\citep{kearns1994learnability}.

\paragraph{Organization of the paper}

We explain Barron's original theorem in Section \ref{sec:barron}, our criterion for representation by multi-layer neural networks in Section \ref{sec:rep}, and give our separation result in Section \ref{sec:separation}. Most proofs and background on Fourier analysis are left in Appendix.

\subsection{Notation and Definitions}
First, we formally define the model of a feedforward neural network that we will use.
\begin{df}
A \vocab{neural network with $n$ hidden layers} (also referred to as a $n+1$-layer neural network) is defined as follows. 
A neural network has an associated input space $\R^{m_0}$, output space $\R^{m_{n+1}}$, and $n$ hidden  layers of sizes $m_1,\ldots, m_{n} \in \N$.
The neural network has parameters $A^{(l)}\in \R^{m_{l-1}\times m_l}$ and $b^{(l)} \in \R^{m_l}$ for $1\le l\le n+1$. The neural network has a fixed activation function $\si$, which is applied component-wise on a vector.  
On input $x\in \R^{m_0}$, the network computes
\begin{align}
x^{(0)}:&=x\\
x^{(l)}:&= \si(A^{(l-1)}x^{(l-1)}+b^{(l)})&1\le l\le n\\
x^{(n+1)} :&= A^{(n+1)} x^{(n)} + b^{(n+1)}.
\end{align}
and outputs $x^{(n+1)}$.
This can also be written out in terms of the components:
$$
x^{(l)}_j := \si\pa{\sumo k{m_l} A^{(l-1)}_{jk} x^{(l-1)}_k +b^{(l-1)}_k}. 
$$
\end{df}

Common choices of activation functions $\si$ include the logistic function $\rc{1+e^{-x}}$, $\tanh(x)$, and the ReLU function $\max\{0,x\}$.

\begin{df}
For a function $f\colon\R^m\to \R^n$, define $\Lip(f)=\Lip_2(f)$, the Lipschitz constant of $f$ with respect to the $L^2$ norm, by
$$
\inf\set{C}{\forall x,y, \ve{f(x)-f(y)}_2\le C\ve{x-y}_2}.
$$
\end{df}

Let $B_n$ be the unit ball in $n$ dimensions$\set{x\in \R^n}{\ve{x}\le 1}$. For sets $A,B$ and a scalar $r$, let
\begin{equation}
A+B:= \set{x+y}{x\in A, y\in B}, \quad rA:= \set{rx}{x\in A}.
\end{equation}
For example, $rB_n$ denotes the ball of radius $r$ in $n$ dimensions, and $A+rB_n$ is the neighborhood of radius $r$ around $A$.

Let $\ved=\ved_2$ denote the usual Euclidean norm on vectors in $\R^n$. For a function $f$, let $f^{\vee}(x):= f(-x)$. (This notation is often used in Fourier analysis.) Let $f^{(n)}(x) = \dd{{}^n}{x^n}f(x)$ denote the $n$th derivative, and $\De f = \sumo in \pdt{}{x_i}f$ denote the Laplacian.




%% file: barron.tex
\section{Barron's Theorem}
\label{sec:barron}
For $f\in L^1(\R)$ we define the Fourier transform of $f\colon \R^n\to \R$ with the following normalization.
\begin{align}\label{eq:fourier}
\wh f(\om) :=\rc{(2\pi)^n}\int_{\R^n} f(x)e^{-i\an{\om, x}}\,dx.
\end{align}

For vector-valued functions $f\colon \R^n\to \R^m$, define the Fourier transform componentwise.

The inverse Fourier transform is 
\[
(\cal F^{-1} g)(x) :=\int_{\R^n} g(\om) e^{i\an{\om, x}}\,dx = (2\pi)^n\wh{g}^{\vee}
\]

The Fourier inversion formula, which holds for all sufficiently ``nice'' functions, is
\[
f(x) = \int_{\R^n} \wh f(x)e^{i\an{\om, x}}\,dx.
 = (2\pi)^n\hat{\hat f}^{\vee}
\]
For background on Fourier analysis with rigorous statements, see Appendix~\ref{sec:fourier}.

\cite{Barron1993} defines a norm on functions defined on a set $B$, and shows that a small norm implies that the function is amenable to approximation by a neural network with one hidden layer.

\begin{df}
For a bounded set $B\subeq \R^p$ 
let $\ve{\om}_B=\sup_{x\in B} |\an{\om, x}|$. 
For a function $f\colon \R^n\to \R$, define the norm
$
\ve{f}_B^* :=\int_{\R^n} \ve{\om}_B |\wh f(\om)|\,d\om.
$
\end{df}

When $B=B_n$ is the unit ball, $\ve{\om}_B = \ve{\om}_2$. 
In this case,  using Theorem \ref{thm:fderiv},
\[
\ve{f}_B^* = 
\int_{\R^n}\ve{\om} |\wh f(\om)|\,d\om
=
\ve{\ve{\om \wh f}_2}_1 = \ve{\ve{\wh{\nb f}}_2}_1
\]
where for a function $g:\R^n\to \R^n$, 
$\ve{g}_2$ is thought of as a function $\R^n\to \R$, and $\ve{\ve{g}_2}_1$ is the $L^1$ norm of this function.

We would like to define this norm for functions $f\colon B\to \R$. However, the Fourier transform is defined for functions $f\colon \R^n\to \R$. Because we only care about the value of $f$ on $B$, we allow arbitrary extension outside of $B$.
 \begin{df}
 Let $B\subeq \R^n$. 
 Let $\cal F_B$ be the set of functions for which the Fourier inversion formula holds on $B$ after subtracting out $g(0)$:\footnote{This is a strictly larger set than functions for which the Fourier inversion formula holds.}
 $$
 \cal F_B = \set{g:\R^n\to \R}{\forall x\in B, g(x)=g(0) + \int (e^{i\an{\om, x}}-1)\wh{g}(\om)\,d\om}.
 $$
 
 Define $\Ga_B=\set{f\colon B\to \R}{\exists g, g|_B=f, g\in \cal F_B}$, let $\Ga_{B}(C)$ be the subset with norm $\le C$ $\Ga_{B}(C)=
\set{f\colon B\to \R}{\exists g, g|_B=f ,\ve{g}_B^*\le C, g\in \cal F_B}$.
We say that a function $f\in \Ga_{B}(C)$ is \vocab{$C$-Barron} on $B$.
For a function $f\colon B\rightarrow \mathbb{R}$, let $C_{f,B}$ be the minimal constant for which $f\in \Ga_{B,C}$: 
\begin{align}\label{eq:barron-constant}
C_{f,B} := \inf_{g|_B=f, g\in \cal F_B} \int_{\R^n} \ve{\om}_B|\wh g(\om)|\,d\om.
\end{align}
When the set $B$ is clear, we just write $C_f$.
\end{df}
This definition is non-algorithmic. How to compute or approximate the Barron constant in general is an open problem. The difficulty stems from the fact that we have to take an infimum over all possible extensions. The Barron constant can be upper-bounded by choosing any extension $f$, but is more difficult to lower-bound. We will give a technique to lower-bound the Barron constant in Theorem~\ref{thm:barron-lb}.

We give some intuition on the Barron constant. First, in order for the Barron constant to be finite, $f$ must be continuously differentiable. Indeed, the inverse Fourier transform of $\om \wh f(\om)$ is $-i\nb f(x)$, and integrability of a function implies continuity of its (inverse) Fourier transform, so $\nb f$ is continuous.

Second, 
the Barron constant will be larger when $\wh f$ is more ``spread out.'' One can think of $\ve{g}_B$ as a kind of $L^1$ norm. This makes sense in the context of neural networks, because if $f(x)=\sumo ik c_i \si(\an{a_i,x}+b_i)$ then $f$ has Fourier transform completely supported on the lines in the direction of the $a_i$.\footnote{Here $f$ does not approach 0 as $\ve{x}\to \iy$, so the Fourier transform must be understood in the sense of distributions.} One can think of the Barron constant as a $L^1$ relaxation of this ``sparsity'' condition.

Barron's Theorem gives an upper bound on how well a function can be approximated by a neural network with 1 hidden layer of $k$ nodes, in terms of the Barron constant. 

For a list of functions with small Barron constant, as well as the effect of various operations on the Barron constant, see \cite[\S IX]{Barron1993}. Examples of Barron functions include polynomials of low degree, ridge functions, and linear combinations of Barron functions.

\begin{df}
A sigmoidal function is a bounded measurable function $f\colon \R\to \R$ such that \\$\lim_{x\to -\iy}f(x)=0$ and $\lim_{x\to \iy}f(x)=1$. 
\end{df}

\begin{thm}[Barron, \cite{Barron1993}]\label{thm:barron}
Let $B\subeq \R^n$ be a bounded set, and $\mu$ any probability measure on $B$.
Let $f\in \Ga_{B}(C)$ and $\si$ be sigmoidal. There exist $a_i\in \R^n$, $b_i\in \R$, $c_i\in \R$ with $\sum_{i=1}^k |c_i|\le 2C$ such that letting $
f_k(x)=\sum_{i=1}^k c_i\si(\an{a_i,x}+b_i)$,
we have
\[
\ve{f-f_k}_\mu^2 :=\int_B (f(x)-f_k(x))^2\,\mu(dx) \le \fc{(2C)^2}{k}.
\]
\end{thm}
Barron's Theorem works for the logistic function (which is sigmoidal), hyperbolic tangent (which is sigmoidal if rescaled to $[0,1]$), and ReLU up to a factor of 2 in the number of nodes. Even though the ReLU function $\text{ReLU}(x) = \max\{0,x\}$ is not sigmoidal, the linear combination $\text{ReLU}(x) = \text{ReLU}(x) - \text{ReLU}(x-1)$ is.

Note that  Barron's Theorem doesn't give approximability tailored to a specific measure $\mu$; it simultaneously gives approximability for \emph{all} $\mu$ defined on $B$, and up to any degree of accuracy. This is why some degree of smoothness is necessary for $f$: otherwise, $\mu$ could be concentrated on the regions where $B$ is not smooth.
Note that approximability for all $\mu$ will be crucial to the proof of the main theorem (Theorem \ref{cor:multi}).
\footnote{Although Barron's Theorem seems to require a strong smoothness assumption, we can approximate any continuous function arbitrarily well with a smooth function and then apply Barron's Theorem.\\
A converse to Barron's Theorem cannot hold in the form stated, because if $\ve{a_i}$ is not restricted, then $\si(\an{a_i,x}+b_i)$ could have large gradient; the Barron constant of $\phi(\an{a_i,x}+b_i)$ would scale as $\ve{a_i}$.\\
It is natural to ask whether we can choose the $a_i$ to have bounded norm. Barron \cite[Theorem 3]{Barron1993} shows a version of the theorem that produces a representation with $\ve{a_i}\le \tau$, but that incurs an additive error $C_\tau$ in the approximation.\\
Note that the following weak converse holds: the Barron constant of $f = c_0+ \sumo ir c_i \si(\an{a_i,x}+b_i)$ is bounded by $O(\diam(K)\sumo ir |c_i|\ve{a_i})$.}

%% file: body.tex
\section{Multilayer Barron's Theorem}
\label{sec:rep}

\subsection{Main theorem}

Barron's Theorem says that a Barron function can be approximated by a neural net with 1 hidden layer. From this, it is reasonable to suspect that a composition of $l$ Barron functions can be approximated by a neural network with $l$ hidden layers. 
Our main theorem says that this is the case; we give a sufficient criterion for a function to be approximated by a neural network with $l$ hidden layers, on any distribution supported in a fixed set $K_0$.

We note two caveats: first, $f_i$ need to be Lipschitz to prevent the error from blowing up. Second, we will need our functions $f_i$ to be Barron on a slightly expanded set (assumption \ref{item:barron}), because an approximation $g_i$ to $f_i$ could take points outside $K_i$, and we need to control the error for those points. 

Given a sequence of functions $f_i$ and $j\ge i$, let $f_{j:i} := f_j\circ f_{j-1}\circ \cdots \circ f_i$.

\begin{thm}[Main theorem]\label{cor:multi}
Let $\ep, s>0$ be parameters, and $l\ge 1$. For $0\le i\le l$ let $m_i\in \N$. Let $f_i: \R^{m_{i-1}}\to \R^{m_i}$ be functions, $\mu_0$ be any probability distribution on $\R^{m_0}$, and $K_i\sub \R^{m_i}$ be sets.

Suppose the following hold.
\begin{enumerate}
\item \label{item:init-supp}
(Support of initial distribution) $\Supp(\mu_0)\sub K_0$.
\item \label{item:lip} ($f_i$ is Lipschitz) $\Lip(f_i)\le 1$.
\item \label{item:barron}
($f_i$ is Barron)
$f_1\in \Ga_{K_0}(C_0)$ and for $1\le i\le l$, 
$f_i\in \Ga_{K_{i-1} + sB_{m_{i-1}}}(C_i)$.
\item 
\label{item:chain}
($f_i$ takes 
each set to the next) 
$f_i(K_{i-1} 
)\subeq K_{i}$
\end{enumerate}
Suppose that the diameter of $K_l$ is $D$. 
Then there exists a neural network $g$ with $l$ hidden layers with $\ce{\fc{4C_i^2m_i}{\ep^2}}$ nodes on the $i$th layer,
so that
\begin{align}\label{eq:main}
\pa{\int_{K_0} 
\ve{f_{l:1}-g}^2\,d\mu_0}^{\rc2}&\le 
l \ep 
\sqrt{
(2C_l\sqrt{m_l} + D)^2
\fc{l}{3s^2}+1}.
\end{align}
\end{thm}
We prove this in Section~\ref{sec:proof}. It is crucial to the proof that Barron's Theorem simultaneously gives approximability for \emph{all} probability distributions on a given set. 

Note that if $K_{l-1}$ is a ball of radius $r$, by the way we defined the norm $\ved_{K_{l-1}}$ in the Barron constant, $C_l$ will at least scale as $s+r$. If we set $s$ to be on the same order as $r$, then the RHS of~\eqref{eq:main} is on the order of $l^{\fc 32} m_l^{\rc 2}\ep$.

\subsection{Approximating probability distributions}

Theorem~\ref{cor:multi} can be interpreted in a very natural way when the aim is to approximate the probability distribution $f_{l:1}(x), x\sim \mu_0$. The Wasserstein distance is a natural distance defined on distributions.

\begin{df}
Let  $\mu,\nu$ be two probability distributions on $\R^n$. Let $\Ga(\mu,\nu)$ denote the set of probability distributions on $\R^n\times \R^n$ whose marginals on the first and second factors are $\mu$ and $\nu$ respectively. (A distribution $\ga\sim \Ga(\mu,\nu)$ is called a \vocab{coupling} of $\mu$, $\nu$.) For $1\le p<\iy$, define the $p$th \vocab{Wasserstein distance} by
$$
W_p(\mu,\nu) = \pa{
\inf_{\ga\in \Ga(\mu,\nu)} \int_{\R^n\times \R^n} \ve{x-y}_2^p \,d\ga(x,y)
}^{\rc p}
$$

\end{df}
When $p=1$, this is also known as the ``earth mover's distance.'' One can think of it as the minimum ``effort'' required to change the distribution of $\mu$ to that of $\nu$ by shifting probability mass (where ``effort'' is an integral of mass times distance).

\begin{cor}\label{cor:dist}
Keep the notation in Theorem~\ref{cor:multi} and suppose the diameter of the set $f_{l:1}(K_0)$ is $D$.
Then 
the Wasserstein distance between the distribution $
f_{l:1}(X) (X\sim \mu_0) 
$
 and
$
g(X), (X\sim \mu_0) 
$
is at most $
l \ep 
\sqrt{1 + 
(2C_l\sqrt{m_l} + D)^2
\fc{l}{3s^2}}$.
\end{cor}

The proof of this is simple: observe that $(f_{l:1}(X), g(X))$, $X\sim \mu_0 
$ defines a coupling between the distributions. Thus by Theorem \ref{cor:multi} the $W_2$ Wasserstein distance is at most 
$$
\ba{\EE_{X \sim \mu_0
}  \ve{f_{l:1}(X)- g(X)}^2 }^{\rc 2}  \le
l \ep 
\sqrt{ 
(2C_l\sqrt{m_l} + D)^2
\fc{l}{3s^2} +1}. 
$$

The Wasserstein distance is a suitable metric in the context of GANs (\cite{arjovsky2017towards}, \cite{arjovsky2017wasserstein}). One way to model a discriminator is as a function $f$ in a certain class $F$ that maximizes the difference between $\E f$ on the real distribution $\mu$ and the generated distribution $\nu$,
\begin{align}\label{eq:mmd}
\sup_{f\in F}\ab{ \EE_{x\sim \mu} f(x) - \EE_{y\sim \nu} f(y)}.
\end{align}
 This is called the maximal mean discrepancy (\cite{kifer2004detecting}, \cite{dziugaite2015training}). 
 The Wasserstein distance captures the idea that if two distributions are close, then it is hard for such a Lipschitz discriminator to tell the difference, as the following lemma shows.

\begin{lem}[Properties of Wasserstein metric]
\label{lem:prop-W}
For any two distributions $\mu, \nu$ over $\R^n$, $W_1(\mu, \nu)\le W_2(\mu, \nu)$.
Moreover, for any Lipschitz function $f\colon \R^n\to \R$,
\begin{align}
\ab{\EE_{x\sim \mu} f(x) - \EE_{y\sim \nu} f(y)}\le \Lip(f) W_1(\mu,\nu).
\end{align}
\end{lem}
%

Proof is deferred to Appendix~\ref{sec:was}. In the context of Corollary~\ref{cor:dist}, Lemma~\ref{lem:prop-W} says that the distribution generated by $f_{l:1}$ and by the neural network cannot be distinguished by a Lipschitz function. \cite{arjovsky2017wasserstein} discuss why the class of Lipschitz functions is a good choice in comparison to other classes. For instance, if we maximize over the class of indicator functions (of measurable sets) instead,~\eqref{eq:mmd} becomes the total variation (TV) distance, which is unstable under perturbations to the function generating the distribution. In particular, the TV distance is discontinuous under perturbations of distributions supported on lower-dimensional subsets of the ambient space $\R^n$.

\subsection{Proof of main theorem}
\label{sec:proof}
To prove Theorem~\ref{cor:multi} we first prove the following theorem. 

\begin{thm} 
\label{thm:multi-barron}
Keep conditions 1--4 and the notation of Theorem~\ref{cor:multi}.
Then there exists a neural network $g$ with $l$ hidden layers and $S\sub \R^{m_0}$ satisfying $\mu_0(S) \ge 1-\pa{\sum_{i=1}^{l-1}{i^2}}\fc{{\ep}^2}{s^2}$ so that
\begin{align}
\pa{\int \one_{S} 
\ve{f_{l:1}-g}^2\,d\mu_0}^{\rc2}\le 
l \ep
\end{align}
\end{thm}

\begin{proof}
Let $r_i = \ce{\fc{4C_i^2m_i}{\ep^2}}$. 
We will show that we can take $g = g_{l:1}$, where
$g_1,\ldots, g_l$ are functions defined by
\begin{align}
g_i&:\R^{m_{i-1}} \to \R^{m_i}\\
(g_i(x))_j &= c_{ij0} + \sum_{k=1}^{r_i} c_{ijk} \si(\an{a_{ijk}, x} + b_{ijk}),
\end{align}
for some parameters $c_{ijk}, b_{ijk}\in \R$, $a_{ijk}\in \R^{m_{i-1}}$. Note that each $g_i$ is a neural net with one hidden layer and a linear output layer. When the next layer $g_{i+1}$ is applied to the output $y$ of $g_i$, first linear functions $\an{a_{i+1,j,k}, y}+b_{i+1,j,k}$ are applied; these linear functions can be collapsed with the linear output layer of $g_i$. Thus only one hidden layer is added each time.

We prove the statement by induction on $l$. For $l=1$, the theorem follows directly from Barron's Theorem~\ref{thm:barron}, using assumptions \ref{item:init-supp} and \ref{item:barron}.

For the induction step, assume we have functions $g_1,\ldots, g_{l-1}$ satisfying the conclusion for $f_1,\ldots, f_{l-1}$. Let $S_{l-1}$ be the set in the conclusion. 
 Apply Barron's Theorem~\ref{thm:barron} to $f_{l}$ to get that that for each $1\le j\le m_l$, for any $\mu$ supported on a set $K_{l-1}'\subeq \R^{m_{l-1}}$ and any $r_l\in \N$, 
there exists a neural net $g_{l,j}$ with 1 hidden layer with $r_l$ nodes such that 
\[
\pa{\int_{\R^{m_{l-1}}} \ba{(f_{l})_j - (g_{l})_j}^2\,d\mu}^{\rc 2} \le \fc{2C_{f_l,K_{l-1}'}}{\sqrt{r_l}}.
\]
Note it is vital here that Barron's Theorem applies to \emph{any} distribution $\mu$ supported on $K_{l-1}'$. 
Let $S_l = S_{l-1}\cap \set{x}{g_{l-1:1}(x)\in K_{l-1}+s B_{m_{l-1}}}$.  
Apply Barron's Theorem with $K_l' = K_l + s B_{m_l}$, $r_l = \ce{\fc{4C_l^2 m_l}{\ep^2}}$.
$\mu = g_{l-1:1*}(\one_{S_l}\mu_0)$.
\footnote{The pushforward of a measure $\mu$ by a function $f$ is denoted by $f_*\mu$ and defined by
$f_*\mu(S) = \mu(f^{-1}(S))$. 
Here, $g_{l-1:1*}(
\one_{S_l}
\mu_0)(S) =  \mu_0(g_{l-1:1}^{-1}(S)\cap S_l)$.
}
We have that $\mu$ is supported on $g_{l-1:1}(S_l)\subeq K_{l-1} + s B_{m_{l-1}} = K_{l-1}'$, as required, and $f_l$ is $C_l$-Barron on this set by assumption \ref{item:barron}.
(Note that $\mu$ is not a probability measure because it was restricted to the set $g_{l-1:1}(S_l)$, but it is a nonnegative measure with  total $L^1$ mass at most 1. Because Barron's Theorem holds for any probability measure, it also holds for these measures.) 
The conclusion of Barron's Theorem gives $(g_l)_j$ such that 
\begin{align}
\pa{\int_{\R^{m_{l-1}}} [(f_l)_j - (g_l)_j]^2 \,d(g_{l-1:1*}(
\one_{S_l}
\mu_0))}^{\rc 2} &\le \fc{2C_l}{\sqrt{r_l}} \le 
\fc{\ep}{\sqrt{m_l}}\\
\implies
\pa{\int_{\R^{m_{l-1}}} \ve{f_l - g_l}^2 \,d(g_{l-1:1*}(
\one_{S_l}
\mu_0))}^{\rc 2} &\le \ep
\end{align}

We bound by the triangle inequality
\begin{align*}
&\quad
\pa{\int_{\R^m}
\one_{S_l} 
\ve{f_{l:1} - g_{l:1}}^{2}\,d\mu_0}^{\rc 2}\\
&\le 
\pa{\int_{\R^m} 
\one_{S_l}
\ve{f_l \circ f_{l-1:1} - f_l \circ g_{l-1:1}}^{2}\,d\mu_0}^{\rc2}
+
\pa{\int_{\R^{m}}
\one_{S_l}
\ve{f_l \circ g_{l-1:1} - g_l \circ g_{l-1:1}}^{2}\,d\mu_0}^{\rc2}\\
&\le 
\pa{\int_{\R^m} 
\one_{S_l}
\ve{f_l \circ f_{l-1:1} - f_l \circ g_{l-1:1}}^{2}\,d\mu_0}^{\rc2}
+
\pa{\int_{\R^{m_{l-1}}} \ve{f_l  - g_l}^{2}\,dg_{l-1:1*}(
\one_{S_l}
\mu_0)}^{\rc2}\\
&\le 
\Lip(f_l)\pa{\int_{\R^m} 
\one_{S_l}
 \ve{(f_{l-1:1} - g_{l-1:1})}^{2}\,d\mu_0}^{\rc 2}
+
\ep\\
&\le 
\Lip(f_l)\pa{\int_{\R^m} 
\one_{S_{l-1}}
 \ve{(f_{l-1:1} - g_{l-1:1})}^{2}\,d\mu_0}^{\rc 2}
+
\ep\\
& \le 1\cdot (l-1)\ep + \ep = l\ep
\end{align*}
The last inequality holds by assumption \ref{item:lip} and the induction hypothesis.

To finish, we have to check that $\mu_0(S_l)\ge 1-\pa{\sumo i{l-1} i^2}\fc{\ep^2}{s^2}$.
As above, we have that 
$${\int 
\one_{S_{l-1}}
\ve{f_{l-1:1} - g_{l-1:1}}^2\,d\mu_0} \le(l-1)^2 \ep^2$$ by the induction hypothesis. Also, $f_{l-1:1}(x)\in K_{l-1}$ for all $x\in \Supp(\mu_0)$ by assumption \ref{item:chain}. Thus by Markov's inequality and the induction hypothesis on $S_{l-1}$,
\begin{align*}
&\quad \mu_0(S_{l-1}\cap\set{x}{g_{l-1:1}(x)\nin K_{l-1} +s B_{m_{l-1}}}) \nonumber\\
&\le 
\mu_0(S_{l-1}\cap \set{x}{\ve{f_{l-1:1}(x) - g_{l-1:1}(x)} \ge s})
\le \fc{(l-1)^2\ep^2}{s^2}
\end{align*}
Therefore $\mu_0(S_l) \ge \mu_0(S_{l-1}) - \fc{(l-1)^2\ep^2}{s^2} \ge 1-\pa{\sumo i{l-1}i^2} \fc{\ep^2}{s^2}$.
\end{proof}

It is inelegant to have to exclude the sets $S_l$. The main theorem is a statement that doesn't involve the sets $S_l$. We achieve this by using the trivial bound on $S_l^c$.

\begin{proof}[Proof of Theorem \ref{cor:multi}]
The functions $g_1,\ldots, g_l$ in Theorem \ref{thm:multi-barron} satisfy $
\int_{S_l} \ve{f_{l:1}-g_{l:1}}^2\,d\mu_0\le l^2\ep^2
$. 
The range of $g_l=((g_l)_1,\ldots, (g_l)_{m_l})$ is contained in a set of diameter $2C_l\sqrt{m_l}$ because the function $\si$ has range contained in $[0,1]$ and Barron's Theorem gives functions $(g_l)_j$, $1\le j\le m_l$, with 
$\sumo kr |c_{ljk}|\le 2C_l$. 

Choose a constant vector $k$ to minimize $\int_{S_l}
\ve{f_{l:1}(x)-g_{l:1}(x)-k}^2\,d\mu_0
$
and replace $g_l$ with $g_l+k$.
Note that now, the range of $g_l$ and $f_l$  necessarily overlap; otherwise a further translation will decrease this error. We still have $\int_{S_l}\ve{f_{l:1}-g_{l:1}}^2\,d\mu_0\le l^2\ep^2$.
Moreover, $\ve{g_l(x)-f_l(x)}\le 2C_l\sqrt{m_l} + D$ for any $x\in K_0$.

Now we have (using $\mu_0(S_l^c) \le\pa{ \sumo i{l-1}i^2} \fc{\ep^2}{s^2}\le \fc{l^3\ep^2}{3s^2}$)
\begin{align}
{\int_{K_0}
\ve{f_{l:1}-g_{l:1}}^2\,d\mu_0}
&\le 
\int_{S_l}
\ve{f_{l:1}-g_{l:1}}^2\,d\mu_0
+
{\int_{S_l^c}
\ve{f_{l:1}-g_{l:1}}^2\,d\mu_0}\\
&\le 
l^2\ep^2+ 
(2C_l \sqrt{m_l}+ D)^2
 \fc{l^3\ep^2}{3s^2}.
\end{align}
Taking square roots gives the theorem.
\end{proof}

%% file: separation.tex
\section{Separation between Barron functions and composition of Barron functions}
\label{sec:separation}

In this section we produce an explicit function $f\colon\R^n\to \R$ that is a composition of two $\poly(n)$-Barron functions, but is not $O(c^n)$-Barron for some $c>1$. 

\begin{thm}
\label{thm:separation}
For any $n\equiv 3\pmod 4$ and $c>1$, there exists a function $f$ and $C_2>0$ 
such that 
\begin{enumerate}
\item ($f$ is not Barron)
$C_{f, C_2n B_n} \ge c^n$.
\item ($f$ is the composition of 2 Barron functions)
$f=j\circ k$ where for all $r,s>0$, $k:\R^n\to \R$ is $O(nr^3)$-Barron on $rB_n$, and $j:\R\to \R$ is $O(sn^2)$-Barron on $sB_1$.
\end{enumerate}
\end{thm}
The condition $n\equiv 3\pmod 4$ is not necessary; we include it only to avoid case analysis.

Note that this theorem gives a separation between Barron functions and compositions of Barron functions, and does not give a separation between distributions expressible by Barron functions and compositions of Barron functions. The analogous question for distributions is an open problem.

We will choose $f$ to be a certain radial function $f=f_1(\ve{x})$ defined in Section~\ref{sec:f}.\footnote{For any radial function $a:\R^n\to \R$, we write $a_1:\R\to \R$ for the function such that $a(x)=a_1(\ve{x})$.} In order for $f$ to have large Barron constant, it is necessary for $\int_{\R^n} \ve{\om}_2|\wh f(\om)|\,d\om$ to be large, i.e. for $\wh f$ to have significant mass  far away from the origin. 
We ensure this holds by choosing $f$ to change sharply in the radial direction. This means $\wh f$ has mass far away from the origin. Moreover, $\wh f$ is radial because $f$ is radial, so $\wh f$ has significant mass in a large shell. 

However, lower-bounding $\int_{\R^n} \ve{\om}_2|\wh f(\om)|\,d\om$  is not sufficient because the definition of the Barron constant requires us to bound this quantity over all extensions of $f$.

To solve this problem, we give a technique to lower bound the Barron constant in Section~\ref{sec:barron-lb} (Theorem~\ref{thm:barron-lb}). Although we cannot certify $f$ is Barron by  showing $\int_{\R^n} \ve{\wh{\nb f}(\om)}\,d\om = \int_{\R^n} \ve{\om}_2|\wh f(\om)|\,d\om$ is large, it suffices to show $\int_{\R^n} \ve{\wh{(\nb f) g}(\om)}\,d\om$ is large for a judiciously chosen $g$. We use this to show that $f$ is not Barron in Section~\ref{sec:f-lb} (Theorem \ref{thm:f-not-barron}). 


We will see in Section \ref{sec:comp} (Theorem \ref{thm:comp}) that $f$ is a composition of two Barron functions $x\mapsto \ve{x}^2$ and $y\mapsto f_1(\sqrt y)$. The function $x\mapsto \ve{x}^2$ is Barron because it 
is a polynomial. The function $y\mapsto f_1(\sqrt y)$ is a function in 1 variable, and it is much easier for a 1-dimensional function $h$ to be Barron as bounds on $h$, $h'$, and $h''$ suffice (Lemma~\ref{lem:1d}).

Our result is similar to the construction in \cite{eldan2015power} of an explicit function that can be approximated by a 3-layer neural net but cannot be approximated (to better than constant error) by any 2-layer neural net with subexponential number of units.
\cite{eldan2015power} use a different Fourier criterion in order to prove a certain function is not computable by a two-layer neural network. 

Roughly speaking, Eldan and Shamir implicitly show that for a specific probability measure that they chose ($\ph^2$, where $\wh \ph = \one_{R_nB_n}$, where $R_n$ is chosen so that $\Vol(R_nB_n)=1$), a necessary criterion for $f$ to be approximated by a 2-layer neural network with $k$ nodes is that most of its mass is concentrated in $k$ ``tubes'' $\bigcup_{i=1}^k(\spn\{v_i\} + R_nB_n)$. (See \cite[Proposition 13, Claim 15, Lemma 16]{eldan2015power}.) The idea can be adapted to other measures. The main difference from Barron's Theorem is that their criterion is a necessary condition for approximability (so useful to show lower bounds), is measure-specific (rather than agnostic to the measure), and is more similar to a ``sparsity'' condition than a ``$L^1$ measure'' as in Barron's Theorem.

\subsection{Definition of $f$}
\label{sec:f}

Let $f_1: \R\to \R$ be a function such that $f_1$ is nonnegative, $\Supp(f_1) \subeq [K_1,K_1+\ep]$, 
$\iiy f_1(x)\,dx=1$, and $|f_1^{(i)}| = O\prc{\ep^{i+1}}$ for all $i=0,1,2$. This function exists by  Lemma~\ref{lem:test}(1).
We will choose $K_1,\ep$ depending on $n$.

By Theorem~\ref{thm:ft-radial}, 
\begin{align}
\wh{f}(\om) &=\rc{2\pi} \prc{2\pi\ve{\om}}^{\fc{n}2-1} \iiy r^{\fc n2-1} f_1(r) J_{\fc n2-1}(\ve{\om}r)\,dr.\label{eq:hatf}
\end{align}
We will choose $[K_1,K_1+\ep]$ to be an interval on which $J_{\fc n2}(\ve{\om} r)$ is large and positive for some large $\ve{\om}$.

We use the notation of Lemma~\ref{lem:besapprox}.
For  $x\ge n$, 
$$
(f_{n,x}x)' = \fc{x}{\sqrt{x^2-\pf{n^2-1}4}} - \fc{\sqrt{n^2-1}}2 \cdot \rc{\sqrt{1-\fc{n^2-1}{4x^2}}}\cdot \fc{-\sqrt{n^2-1}}{2x^2} = 
\sqrt{1-\fc{n^2-1}{4x^2}}\in \ba{\sfc 34,1}.
$$
Let $K_3=C_3\sqrt n$ for some $C_3$ to be chosen. 
In every interval of length $\ge \fc{4\pi}{K_3\sqrt{3/4}}$ there is an interval of length $\ge \fc{\pi}{K_3}$ on which 
\begin{align}\label{eq:cos}
\cos\pa{-\fc{(n+1)\pi}4 + f_{d,K_3r}K_3r}\ge \rc{\sqrt 2}.
\end{align}
Let $[K_1,K_1+\ep]$ be the first such interval with $K_1\ge C_1\sqrt n$, where $C_1$ is a constant to be chosen. Note we have $K_1 \sim C_1\sqrt n$ and $\ep=\Te\prc{K_3}$.

\subsection{A technique to lower bound the Barron constant}
\label{sec:barron-lb}

The main difficulty in showing a function is not Barron is to lower bound the integral
$$\int_{\R^n}\ve{\om}|\wh F(\om)|\,d\om = \int_{\R^n} \ve{\wh{\nb F}(\om)}\,d\om$$ 
over \emph{all} extensions $F$ of $f$. In general, it is not known how to calculate the infimum over all extensions. 

Theorem \ref{thm:barron-lb} gives us a way to lower-bound the Barron constant for $f$ over a ball $rB_n$.
The idea is the following.
Instead of bounding $\int_{\R^n}\ve{ \wh{\nb F}(\om)}\,d\om$ for every extension $F$, we choose $g$ with support in $B$ and compute $\int_{\R^n}\ve{\wh{(\nb F)g}(\om)}\,d\om$. This does not depend on the extension $F$ because $(\nb F)g=(\nb f)g$. It turns out that we can bound $\int_{\R^n} \ve{\wh{\nb F}(\om)}\,d\om$ in terms of  $\int_{\R^n} \ve{\wh{(\nb F)g}(\om)}\,d\om$.

\begin{thm}\label{thm:barron-lb}
If 
 $f$ is differentiable, then for any  $g$ such that $\Supp(g)\subeq rB_n$ and $g,\wh g\in L^1(\R^n)$, 
$$
C_{f,rB_n} \ge r \frac{\int_{\R^n} |\widehat{(\nabla f)g}(\om)|\,d\om}{\int_{\R^n} |\widehat g(\om)|\,d\om}
$$
\end{thm}
Note that $g$ is a function that we are free to choose. 
To use the theorem we will choose $g$ with $\Supp(g)\subeq C_2n B_n$ and $\int_{\R^n} |\wh g(\om)|\,d\om$ small. This theorem is similar to \cite[\S IX.11]{Barron1993}, which bounds the Barron constant of a product of two functions. We defer the proof to Appendix~\ref{sec:separation-proof}.

To use this bound for a function $f$, we need to judiciously choose the function $g$. 
Let $b$ be the ``bump'' function given by Lemma \ref{lem:test}(3) for $m=\fc{n+1}2$. This function has the properties that $b(x) = 1$ for $x\in [-1,1]$, $b(x) = 0$ for $|x|\ge 2$, and for $k\le m$, $b^{(k)}(x) \le (n+1)^k$.
Let $g_1(x) = b_{(K_2)}(x) = b\pf{x}{K_2}$ and $g(x) = g_1(\ve{x})$ for $K_2=C_2n$, where $C_2$ is a constant to be chosen. 

In Appendix~\ref{sec:separation-proof}, we show the following lemma that bounds the Barron constant for $f$.

\begin{lem} \label{lem:notbarron}
For $n\equiv 3\pmod 4$ and constants $C_1,C_2,C_3$ such that $C_1C_3\ge \fc 32$, $C_2>C_1\ge 1$,   $C_3\ge 1$,
the functions $f,g$ we choose satisfy 
\begin{align}
\int_{\R^n} \ab{\wh g(\om)}\,d\om &= O((5eC_2
)^{\fc n2}),\\
\int_{\R^n} \ve{\wh{(\nb f)g}(\om)}\,d\om
& = 
\Om(C_1^{\fc n2-3} C_3^{\fc n2}n^{-\rc 2} e^{\fc n2}).
\end{align}
As a result the Barron constant $C_{f,2K_2B_n} \ge \Om\pa{2^{-n}C_1^{\fc n2-3} C_3^{\fc{n}2} C_2^{-\pa{\fc n2-1}}
n^{\rc 2}}$.
\end{lem}

Therefore, as long as we choose $C_3$ to be large enough this constant is exponentially large. The constraint that  $n\equiv 3\pmod 4$ is only there to avoid case analysis. We give the proof in Section~\ref{sec:separation-proof}.

\subsection{$h$ is a composition of Barron functions}
\label{sec:comp}

We can write $f$ as the composition of a function that computes the square norm, and a one dimensional function. The Barron constant for both functions can be bounded by polynomials.

\begin{lem}\label{thm:comp}
Suppose that $C_1<C_3$. 
$f$ is the composition of the two functions
\begin{align}
x&\mapsto \ve{x}^2 & \R^n&\to \R\\
y&\mapsto f_1(\sqrt{y})& \R &\to \R.
\end{align}
The function $x\mapsto \ve{x}^2$ satisfies $
C_{\ve{x}^2, rB_n}\le O(nr^3)
$
and the function $y\mapsto f_1(\sqrt y)$ satisfies
$
C_{f_1(\sqrt y), [-s, s]} =O(sC_1^{\fc 12}C_3^{\fc 32}n^2)
$
for any $s$.
\end{lem}

Intuitively, the proof uses the fact that polynomials are Barron, and all ``nice'' one dimensional functions are Barron. We leave the detailed proofs in Section~\ref{sec:separation-proof}. Now it is easy to see the separation:

\begin{proof}[of Theorem~\ref{thm:separation}] 
By Lemma~\ref{lem:notbarron}, we know we can choose $C_3$ large enough so that the Barron constant for $f$ is exponential. On the other hand, by Lemma \ref{thm:comp} we know $f$ is a composition of two Barron functions.
\end{proof}

%% file: appendix.tex
\section{Background from Fourier Analysis}
\label{sec:fourier}

The Fourier transform is defined in \eqref{eq:fourier}.

\begin{thm}[Fourier inversion]
For continuous $f$ such that $f\in L^1(\R^n)$ and $\wh f\in L^1(\R^n)$,
\[
f(x) = \int \wh f(x)e^{i\an{\om, x}}\,dx.
 = (2\pi)^n\hat{\hat f}^{\vee}
\]
\end{thm}

\begin{thm}[Plancherel's Theorem]\label{thm:plancherel}
For $f,g:\R^n\to \C$ such that $f,g\in L^1(\R^n) \cap L^2(\R^n)$, 
$$
\int_{\R^n} f(x) \ol{g(x)}\,dx =\int_{\R^n} (2\pi)^n \wh f(\om) \ol{\wh g(\om)}\,d\om.
$$
\end{thm}

\begin{thm}[Fourier transform of derivative]
\label{thm:fderiv}
For differentiable $f:\R^n\to \R$, $f\in L^1(\R^n)$,
$$
\wh{\nb f}(x) = ix \wh{f}(x). 
$$
For $f:\R^n\to \R$ such that $f, \ve{x}f\in L^1(\R^n)$,
$$
(xf)^{\wedge} = i\nb \wh f(x).
$$
\end{thm}

\begin{thm}[Fourier transform of convolution]
\label{thm:fconv}
For $f,g\in L^1(\R^n)$ 
\begin{align}
\wh{f*g}(x) &= \wh f(\om)\wh g(\om)
\end{align}
For $f,g\in L^1(\R^n)$ with $fg, \wh f, \wh g\in L^1(\R^n)$,
\begin{align}
\wh{fg}(x) &= (\wh f*\wh g)(\om).
\end{align}
\end{thm}

\begin{thm}[Fourier transform of radial function]
\label{thm:ft-radial}
Suppose $f(x) = f_1(\ve{x})$ where $f\in L^1(\R^n)$, $f:\R_{\ge 0}\to \R$. Then
$$
\wh f(\om) =\rc{2\pi} \prc{2\pi\ve{\om}}^{\fc{n}2-1} \iiy r^{\fc n2-1} f_1(r) J_{\fc n2-1}(\ve{\om}r)\,dr.
$$
where $J_\al$ is the Bessel function of order $\al$.
\end{thm}

\begin{lem}[$L^1$ bound on Fourier transform]
\label{lem:1d}\label{lem:l1-f}
$\,$
\begin{enumerate}
\item
Let $k\ge \fc{n+1}2$ and $k$ be even. Then for $g:\R^n\to \R$ that is $k$ times differentiable,
\begin{align}
\int_{\R^n} \ve{\wh g(\om)}\,d\om
&\le \pf{\Ga\prc 2}{2^n\pi^{\fc n2}\Ga\pf{n+1}2}^{\rc 2} \pa{\int_{\R^n} [(I-\De)^{\fc{k}2} g(x)]^2\dx}^{\rc 2}.
\end{align}
\item
Let $h:\R\to \R$ be once or twice differentiable, respectively. Then
\begin{align}
\iny |\wh h(\om)|\,d\om
&\le 2^{-\rc 2}
\pa{\iny |h|^2 + |h'|^2\dx}^{\rc 2}\\
\iny |\om \wh h(\om)|\,d\om
&\le 2^{-\rc 2}
\pa{\iny |h'|^2 + |h''|^2\dx}^{\rc 2}.
\end{align}
\end{enumerate}
\end{lem}
\begin{proof}
By Cauchy-Schwarz and the fact that  $\int_{\R^n}\rc{\pa{1+\ve{\om}^2}^{\fc{n+1}2}}\,d\om = \fc{\pi^{\fc n2}\Ga\prc2}{\Ga\pf{n+1}2}$ (this is used e.g. to define the Cauchy probability distribution)
\begin{align}
\int_{\R^n} \ve{\wh g(\om)}\,d\om
&\le \pa{\int_{\R^n} \rc{\pa{1+\ve{\om}^2}^{k}} \,d\om}^{\rc 2}
\pa{\int_{\R^n} (1+\ve{\om}^2)^{k} |\wh g(\om)|^2\,d\om}^{\rc 2}
\\
&\le
\pa{\int_{\R^n} \rc{\pa{1+\ve{\om}^2}^{\fc{n+1}2}} \,d\om}^{\rc 2}
\pa{\int_{\R^n} (1+\ve{\om}^2)^{k} |\wh g(\om)|^2\,d\om}^{\rc 2}
\\
&\le \pf{\pi^{\fc n2}\Ga\prc 2}{\Ga\pf{n+1}2}^{\rc 2} \pa{\int_{\R^n} \ab{(1+\ve{\om}^2)^{\fc{k}2} \wh g(\om)}^2\,d\om}^{\rc 2}\\
&\le \pf{\pi^{\fc n2}\Ga\prc 2}{\Ga\pf{n+1}2}^{\rc 2}
(2\pi)^{-\fc n2} \pa{\int_{\R^n} [(I-\De)^{\fc{k}2}g(x)]^2\dx}^{\rc 2}
\end{align}
where in the last step we used Theorem~\ref{thm:plancherel} and the calculation
$$
\wh{\De g} = \pa{\sumo in \pdt{}{x_i}g}^{\wedge} = -\sumo in \om_i^2 \wh g(\om) = -\ve{\om}^2\wh g(\om).
$$

For the second part, again by Cauchy-Schwarz and $\wh{h'}(\om) = i\om h(\om)$,
\begin{align}
\iny | \wh h(\om)|\,d\om
& \le\pa{ \iny \rc{1+|\om|^2} \,d\om\iny |\wh h(\om)|^2 (1+|\om|^2) \,d\om }^{\rc 2}\\
&\le\sqrt\pi  \pa{\iny |\wh{h}|^2 + |\wh{h'}|^2 \,d\om}^{\rc 2} \\
&\le \sqrt\pi (2\pi)^{-\fc 12} \pa{\iny |h|^2 + |h'|^2\,dx}^{\rc 2}.
\end{align}
This gives the first equation. To get the second, replace $h$ with $h'$.
\end{proof}

\section{Bessel functions}
\label{sec:bessel}

We will need some facts about Bessel functions $J_\al(x)$, $\al\in \R$. $J_\al(x)$ has an oscillating shape like a damped sinusoid.
%

\begin{lem}[{\cite[Theorem 5]{krasikov2014approximations}, \cite[Lemma 21]{eldan2015power}}]\label{lem:besapprox}
	If $d\geq 2$ and $x\geq d$, then
	\[
	\left|J_{d/2}(x)-\sqrt{\frac{2}{\pi c_{d,x} x}}\cos\left(-\frac{(d+1)\pi}{4}+f_{d,x}x\right)\right| ~\leq~ x^{-3/2},
	\]
	where
	\[
	c_{d,x} = \sqrt{1-\frac{d^2-1}{4x^2}}~~~,~~~ f_{d,x}=c_{d,x}+\frac{\sqrt{d^2-1}}{2x}\arcsin\left(\frac{\sqrt{d^2-1}}{2x}\right).
	\]
	Moreover, assuming $x\geq d$,
	\[
	1\geq c_{d,x} \geq 1-\frac{0.15~d}{x} \geq 0.85
	\]
	and
	\[
	1.3\geq 1+\frac{0.3~d}{x}\geq f_{d,x} \geq 1-\frac{0.15~d}{x} \geq 0.85 .
	\]
\end{lem}

\begin{lem}[{\cite[Lemma 20]{eldan2015power}}]
\label{lem:lip-bessel}
For any $\al\ge 1$ and $x\ge 3\al$, $J_\al(x)$ is 1-Lipschitz in $x$. 
\end{lem}



%
%
\section{Properties of Wasserstein Distance}
\label{sec:was}
\begin{lem}[Lemma~\ref{lem:prop-W} restated]
For any two distributions $\mu, \nu$ over $\R^n$,
\begin{align}\label{eq:w1w2}
W_1(\mu, \nu)\le W_2(\mu, \nu).
\end{align}
Moreover, for any Lipschitz function $f:\R^n\to \R$,
\begin{align}\label{eq:lip-W}
\ab{\EE_{x\sim \mu} f(x) - \EE_{y\sim \nu} f(y)}\le \Lip(f) W_1(\mu,\nu).
\end{align}
\end{lem}
\begin{proof}
Let $\ga\in \Ga(\mu,\nu)$ be a coupling of $\mu, \nu$. Then by the Cauchy-Schwarz inequality,
\begin{align}
W_1(\mu, \nu) &\le 
\int_{\R^n\times \R^n} \ve{x-y}_2\,d\ga(x,y)\\
&\le \pa{\int_{\R^n\times \R^n} \ve{x-y}_2^2\,d\ga(x,y)}^{\rc 2}\ub{\pa{\int_{\R^n\times \R^n}\,d\ga}^2}{1}.\label{eq:w2}
\end{align}
The infimum of~\eqref{eq:w2} over all couplings $\ga\sim \Ga(\mu,\nu)$ is exactly $W_2(\mu,\nu)$. This shows~\eqref{eq:w1w2}.

Now for any $\ga\in \Ga(\mu,\nu)$, because its marginals are $\mu$ and $\nu$,
\begin{align}
\ab{\EE_{x\sim \mu} f(x) - \EE_{y\sim \nu} f(y)}
&=
\ab{\int_{\R^n\times \R^n} f(x) - f(y)\,d\ga(x,y)}\\
&\le \Lip(f) \int_{\R^n\times \R^n}\ve{f(x) - f(y)}_2
\,d\ga(x,y).\label{eq:lip-w1}
\end{align}
The Lipschitz constant is with respect to the $L^2$ norm because we use the $L^2$ norm to measure the distance between $f(x)$ and $f(y)$.
Taking the infimum of \eqref{eq:lip-w1} gives \eqref{eq:lip-W}. 
\end{proof}

In fact, \eqref{eq:lip-W} is sharp when $\mu,\nu$ have bounded support. The duality theorem of Kantorovich and Rubinstein \citep{kantorovich1958space} says that
$$
W_1(\mu,\nu) = \sup\set{
\EE_{x\sim \mu} f(x) - \EE_{y\sim \nu} f(y)}{f:\R^n \to \R, \Lip(f) \le 1}.
$$

%% file: bump_new.tex
\section{Test functions}

For a function $f$, let $f_{(K)} (x) := f\pf xK$. 

\begin{lem}
$\,$
\label{lem:test}
Let $m\ge2$ be a given positive integer.
\begin{enumerate}
\item
There exists a function $g:\R\to \R$ with the following properties. 
\begin{enumerate}
\item
$g\ge 0$ everywhere.
\item
$\Supp(g)\subeq [0,1]$.
\item
$\int_0^1 g(x)\dx = 1$.
\item
$g$ is $m$ times continuously differentiable and for all $k\le m$, $|g^{(k)}(x)| = O((2m)^{k+1})$.
\end{enumerate}
The function $\rc K g_{(K)}(x)$ satisfies $\Supp(g_{(K)}) \subeq [0,K]$, $\int_0^K g_{(K)}\dx=1$, and for $k\le m$, $g_{(K)}^{(k)}(x) =O\pa{\pf{2m}{K}^{k+1}}$.
\item
There exists a function $G:\R\to \R$ with the following properties.
\begin{enumerate}
\item
$G$ is nondecreasing.
\item
$G(x)=0$ for $x\le 0$.
\item
$G(x)=1$ for $x\ge 1$.
\item
$G$ is $m+1$ times continuously differentiable and for all $k\le m$, $G^{(k)}(x) = O((2m)^{k})$. 
\end{enumerate}
\item
There exists a function $b:\R\to \R$ with the following properties:
\begin{enumerate}
\item
$\Supp(b)\subeq [-2,2]$.
\item
$b(x)= 1$ for $x\in [-1,1]$.
\item
$b$ is is $m+1$ times continuously differentiable and for all $k\le m$, $b^{(k)}(x) = O((2m)^k)$. 
\end{enumerate}
The function $b_{(K)}$ satisfies $\Supp(b_{(K)}) \subeq [-2K,2K]$, $b_{(K)}(x) = 1$ for $x\in [-K,K]$, and $b^{(m)}_{(K)}(x) = O\pa{\pf{2m}{K}^k}$.
\end{enumerate}
\end{lem}

\begin{proof}
Take 
$$g(x) = 
\begin{cases}
C_m 4^{m+1} x^{m+1}(1-x)^{m+1}, &x\in [0,1]\\
0, &\text{else.}
\end{cases}
$$
where $C_m$ is chosen so that $\int_0^1 g(x)\dx=1$. Note that $x(1-x)\le \rc 4$ so $g(x)\le C_m$ and
\begin{align}
1=\int_0^1 g(x)\dx
&\le C_m\\
1=\int_0^1 g(x)\dx
&\ge \int_{\rc 2-\rc{2\sqrt m}}^{\rc 2+\rc{2\sqrt m}}
 C_m 4^{m+1} x^{m+1}(1-x)^{m+1}\,dx\\
 &\ge \rc{\sqrt m} C_m 4^{m+1} \pa{\rc 2+ \rc{2\sqrt m}}^{m+1} \pa{\rc 2- \rc{2\sqrt m}}^{m+1}\\
 &\ge \rc{\sqrt m} C_m \pa{1-\rc{m}}^{m+1}\\
 &\ge \fc{C_m}{2e\sqrt m}
\end{align}
so $1\le C_m \le 2e\sqrt m$. 

Now, note that for functions $u,v$, 
\begin{align}
(uv)^{(k)} &= \sumz jk \binom kj u^{(j)} v^{(k-j)}.
\end{align}
Applying this to $x^{m+1}$ and $(1-x)^{m+1}$ and gives that for $0\le x\le 1$,  $k\le m$, 
\begin{align}
|g^{(k)}(x)| &\le C_m\sqrt m\sumz jk \binom kj  (m+1)^j (m+1)^{k-j}\\
&\le O(m(2(m+1))^k)\\
& = O((2m)^{k+1}).
\end{align}

For the second part, take $F(x) = \int_{-\iy}^x f(t)\,dt$. The normalization $\int_0^1 f(x)\dx=1$ ensures $F(x)=1$ for $x\ge 1$, and for $k\le m$, $F^{(k+1)}(x) = f^{(k)}(x) = O((2m)^k)$.

For the third part, define 
$$b(x) = \begin{cases}
0, &|x|>2\\
F(2-|x|), &1\le |x|\le 2\\
1, &|x|<1.
\end{cases}$$

For the rescaled functions, just note that for any function $f$, $f_{(K)}^{(k)}(x) = \rc{K^k} f^{(k)}\pf xK$. 
\end{proof}

%% file: separation-proof.tex
\section{Omitted Proofs in Section~\ref{sec:separation}}
\label{sec:separation-proof}

\begin{thm}[Theorem~\ref{thm:barron-lb} restated]
If 
 $f$ is differentiable, then for any  $g$ such that $\Supp(g)\subeq rB_n$ and $g,\wh g\in L^1(\R^n)$, 
$$
C_{f,rB_n} \ge r \frac{\int_{\R^n} |\widehat{(\nabla f)g}(\om)|\,d\om}{\int_{\R^n} |\widehat g(\om)|\,d\om}
$$
\end{thm}
\begin{proof}
Let $B=rB_n$.
We have
\begin{align}
C_{f,B} &=\inf_{F|_B = f}  \int_{\R^n}  \ve{\om}_B |\wh F(\om)|\,d\om\\
& = r\inf_{F|_B = f}  \int_{\R^n}  \ve{\om}_2 |\wh F(\om)|\,d\om\\
& = r\inf_{F|_B = f} \int_{\R^n}  \ve{\wh{\nb F}(\om)}_2\,d\om.
\end{align} 
Young's inequality and Theorem~\ref{thm:fconv} give
\begin{align}
\int_{\R^n}  \ve{\widehat{\nabla F}(\om)}_2\,d\om \int_{\R^n}  |\widehat g(\om)| \,d\om 
&\ge \int_{\R^n}  \ve{(\widehat{\nabla F} * \widehat g) (\om)}_2\,d\om\\
& = \int_{\R^n}  \ve{\widehat{(\nabla F)g}(\om)}_2\,d\om \\
& = \int_{\R^n}  \ve{\widehat{(\nabla f)g}(\om)}_2\,d\om.
\end{align}
where the last step uses the fact that $\Supp(g)\subeq rB_n$, so $(\nb F)g = (\nb f)g$.
Then 
\begin{align}
\int_{\R^n}  \ve{\widehat{\nabla F}(\om)}_2 \,d\om&\ge \frac{\int_{\R^n}  \ve{\widehat{(\nabla f)g}(\om)}_2\,d\om}{\int_{\R^n}  |\widehat g(\om)|\,d\om}.
\end{align}
\end{proof}

\subsection{$f$ is not Barron}
\label{sec:f-lb}

In this section we prove Lemma~\ref{lem:notbarron}. We first prove the function $g$ we choose gives a small denominator in the lowerbound equation.

\begin{lem}\label{lem:bd-g}
For $n\equiv 3\pmod 4$, 
$$
\int_{\R^n} \ve{\wh g(\om)}\,d\om \le O((5eC_2
)^{\fc n2}).
$$
\end{lem}

To prove this  we will need  bound certain combinations of derivatives of a radial function.
\begin{lem}\label{lem:ide}
Let $f\colon\R^n\to \R^n$ be a radial function with $f(x) = f_1(\ve{x})$. Then for $k\in \N$,  $1\le k \le  \fc{n}4+1$,
\begin{align}\label{eq:ide}
((I-\De)^k f)(x) = \sumr{0 \le i\le 2k,0\le j\le \max\{0,2k-1\}}{i+j\le 2k} \fc{c_{i,j} {n}^{j}f_1^{(i)}(r)}{r^j},\quad r=\ve{x}
\end{align}
for some $c_{i,j}$ with $\sum_{i,j} |c_{i,j}|\le 5^k$.

Here, $(I-\De)f$ denotes $f-\De f$.
\end{lem}
\begin{proof}
We proceed by induction. The case $k=0$ is just $f(x)=f_1(r)$. 
Suppose the statement is true for a given $k\le \fc{n}4$; we show it for $k+1$. Let $(I-\De)^kf$ be given by \eqref{eq:ide}.
We use the formula for the Laplacian of a radial function,
\begin{align}\label{eq:L-rad}
\De f(x) = \fc{n-1}r f_1'(r)  + f_1''(r).
\end{align}
For ease of notation, in the below the arguments of $f$ and $f_1$, which are $x$ and $r$, are omitted.
Then using \eqref{eq:L-rad} and the product rule,
\begin{align}
\label{eq:sum1}
(I-\De)^{k+1} f &=  \sumr{0 \le i\le 2k,0\le j\le \max\{0,2k-1\}}{i+j\le 2k}
c_{i,j} {n}^{j}
\bigg(
\fc{1}{r^j}f_1^{(i)}
+
\fc{n-1}r
\pa{\fc{j}{r^{j+1}} f_1^{(i)} - \rc{r^j} f_1^{(i+1)}}\\
&\quad
+ 
\pa{-\fc{j(j+1)}{r^{j+2}}f_1^{(i)} +\fc{2j}{r^{j+1}}f_1^{(i+1)} - \rc{r^j}f_1^{(i+2)}}\bigg)
\end{align}
The largest derivative of $f_1$ increases by 2 and the power of $r$ increases by 2, except when $k=0$, when the power increases by 1 (from~\eqref{eq:L-rad}). 
Write this as 
$$
\sumr{0 \le i\le 2(k+1),0\le j\le 2k+1}{i+j\le 2(k+1)} \fc{c_{i,j}' {n}^{j}f_1^{(i)}}{r^j}.
$$
A term is identified by the order $f^{(i)}$ that appears and the power $\rc{r^j}$ that appears. For example, the term $c_{i,j}n^{j} \fc{n-1}r\fc{j}{r^{j+1}}f_1^{(i)} = c_{i,j}n^{j+2}\fc{(n-1)j}{n^2} \rc{r^{j+2}} f_1^{(i)}$ in \eqref{eq:sum1} will contribute $ c_{i,j} \fc{(n-1)j}{n^2}$ to  $c_{i,j+2}'$.
Noting  $k\le \fc{n}4$ implies $2k \le \fc{n}2$, we have
\begin{align}
\sum_{i,j} |c_{i,j}'| 
&\le 
\sumr{0 \le i\le 2k,0\le j\le \max\{0,2k-1\}}{i+j\le 2k} |c_{i,j}|
\pa{1+\fc{(n-1)j}{{n}^2} + \fc{n-1}{n} + \fc{j(j+1)}{{n}^2} + \fc{2j}{n} + 1}\\
&\le 
\sum_{i,j}|c_{i,j}| \pa{1 + \rc 2 +1+ \rc 4+ 1+ 1}\\
&\le 5\sum_{i,j}|c_{i,j}|.
\end{align}
This completes the induction step and proves the theorem.
\end{proof}

\begin{proof}[Proof of Lemma~\ref{lem:bd-g}]
By Lemma~\ref{lem:l1-f} with $k=\fc{n+1}2$,
\begin{align}
\int_{\R^n} \ve{\wh g(\om)}\,d\om
&\le \pf{\Ga\prc 2}{2^n\pi^{\fc n2}\Ga\pf{n+1}2}^{\rc 2} \pa{\int_{\R^n} [(I-\De)^{\fc{n+1}4} g(x)]^2\dx}^{\rc 2}.\label{eq:bar1}
\end{align}
Note $ [(I-\De)^{\fc{n+1}4}g(x)]^2$ is nonzero only
on $2K_2B_n$. Then letting $c_{i,j}$ be as in Lemma~\ref{lem:ide} with $k=\fc{n+1}4$, we have
\begin{align}
(I-\De)^{\fc{n+1}4}g(x)
&=
 \sumr{0 \le i\le \fc{n+1}2,0\le j\le \fc{n-1}2}{i+j\le \fc{n+1}2} \fc{c_{i,j} {n}^{j}g_1^{(i)}(r)}{r^j},\quad r=\ve{x}
 \end{align}
 We separate out the one term $g_1(r)$, and bound the derivatives noting that $g_1$ was defined using the bump function $b_{(K_2)}$ in Lemma~\ref{lem:test}. 
 Note that $g_1^{(i)}=0$ for $r<K_2$, so we can take $r\ge K_2$ in the sum.
 \begin{align}
|(I-\De)^{\fc{n+1}4}g(x)|
&\le 
g_1(r) + 
\sumr{1 \le i\le \fc{n+1}2,0\le j\le \fc{n-1}2}{i+j\le \fc{n+1}2} |c_{i,j}| \fc{n^{j} |g_1^{(i)}(r)|}{r^j}\\
&\le g_1(r) +\sumr{1\le i\le \fc{n+1}2,0\le j\le \fc{n-1}2}{i+j\le \fc{n+1}2} |c_{i,j}| \fc{n^{j} O\pf{(n+1)^i}{(C_2n)^i}}{(C_2n)^j}\\
& = O(4^{\fc{n+1}4}
).\\
\end{align}
Noting that the volume of $2K_2B_n$ is $\fc{\pi^{\fc n2}}{\Ga\pa{\fc n2+1}}(2K_2)^n$, 
\begin{align}
\pa{\int_{\R^n} [(I-\De)^{\fc{n+1}4}g(x)]^2\dx}^{\rc 2}
&= O\pa{\pa{\fc{\pi^{\fc n2}}{\Ga\pa{\fc n2+1}}
(2K_2)^n  \pa{5^{\fc{n+1}4}
}^2
}^{\rc 2}}\\
&= O\pa{\pf{\pi^{\fc n2} 2^nC_2^nn^n}{\Ga(\fc n2+1)}^{\rc 2}{5^{\fc{n+1}4}
}}.
\label{eq:bar2}
\end{align}
Combining \eqref{eq:bar1} and \eqref{eq:bar2} and using Stirling's approximation $\Ga(n+1)\sim \sqrt{2\pi n} \pf{n}{e}^n$ gives
\begin{align}
\int_{\R^n} \ve{\wh g(\om)}\,d\om
&\le O\pf{C_2^{\fc n2} n^{\fc n2} 5^{\fc{n+1}4}
}{\Ga\pf{n+1}2^{\rc 2} \Ga\pa{\fc n2+1}^{\rc 2}}\\
&=
O\pa{(5eC_2
)^{\fc n2}}.
\end{align}
\end{proof}

Now we are ready to bound the numerator and finish the proof.

\begin{lem}\label{thm:f-not-barron} 
For $f$ defined as in Section~\ref{sec:f}, $n\equiv 3\pmod 4$,  and constants $C_1,C_2,C_3$ such that $C_1C_3\ge \fc 32$, $C_2>C_1\ge 1$,   $C_3\ge 1$, 
$$C_{f, 2K_3B_n} = \Om\pa{2^{-n}C_1^{\fc n2-3} C_3^{\fc{n}2} C_2^{-\pa{\fc n2-1}}n^{\rc 2}}.
$$
\end{lem}
In particular, this is exponentially large if we choose $C_3$ large enough (i.e. if we make $f$ vary sharply enough).

\begin{proof}
For $\ve{\om}=K_3$, by \eqref{eq:hatf}, \eqref{eq:cos}, and Lemma~\ref{lem:besapprox},
\begin{align}
\wh f(\om) &= \rc{2\pi}\prc{2\pi K_3}^{\fc{n}2-1}
\int_{K_1}^{K_1+\ep} r^{\fc n2-1} f_1(r) J_{\fc{n}2-1}(K_3r)\,dr\\
 &\ge \rc{2\pi}\prc{2\pi K_3}^{\fc{n}2-1}
\int_{K_1}^{K_1+\ep} r^{\fc n2-1} f_1(r) 
\pa{\sfc{2}{\pi K_3r}\rc{\sqrt 2} - (K_3r)^{-\fc 32}}\,dr\\
& \ge \rc{2\pi}\pf{K_1}{2\pi K_3}^{\fc{n-3}2} \sfc{1}{\pi}(1-o(1))
\label{eq:foo1}
\end{align}
where in the last step we used $\int_{K_1}^{K_1+\ep} f_1(r) = 1$.
Now we show that $\wh f$ is also large for $\ve{\om}\approx K_3$. Let $\om, \om_0$ be such that $\ve{\om_0}=K_3$ and $\om\ge \om_0$. 
Then using the fact that $J_{\fc{n}2-1}$ is 1-Lipschitz for $x\ge 3\pa{\fc{n}2-1}$ (Lemma~\ref{lem:lip-bessel}) and $K_3K_1 \ge C_3C_1n\ge \fc{3n}2$,
\begin{align}
|\wh f(\om) - \wh f(\om_0)|
&\le \rc{2\pi} \prc{2\pi K_3}^{\fc{n}2-1}
\int_{K_1}^{K_1+\ep}
r^{\fc n2-1} f_1(r) |J_{\fc n2-1}(\ve{\om}r) - J_{\fc n2}(K_3r)|\,dr\\
&\le \rc{2\pi} \prc{2\pi K_3}^{\fc{n}2-1}
\int_{K_1}^{K_1+\ep}
r^{\fc n2-1} f_1(r) r(\ve{\om}-K_3)\,dr\\
&\le \rc{2\pi} \prc{2\pi K_3}^{\fc{n}2-1}
(K_1+\ep)^{\fc n2}  (\ve{\om}-K_3)\\
&= O\pa{\pf{K_1}{2\pi K_3}^{\fc n2-1} K_1^{\fc{3}2} K_3^{\fc 12} (\ve{\om}-K_3)}
\label{eq:foo2}
\end{align}
By~\eqref{eq:foo1} and~\eqref{eq:foo2}, for $n\ge 3$, there exists $\de$ such that for all $\ve{\om}\in \ba{K_3, K_3+\fc{\de}{K_1^{3/2}K_3^{1/2}}}$, 
\begin{align}
|\wh f(\om)|&=
 \Om\pa{\pf{K_1}{2\pi K_3}^{\fc{n-3}2}}
\end{align}
Then using the fact that the surface area of a sphere in $\R^n$ is $\fc{2\pi^{\fc n2}}{\Ga\pf n2}$,
\begin{align}
\int_{\R^n} \ve{\om}|\hat f(\om)|\,d\om
& = \int_{K_3\le \ve{\om}\le K_3+\fc{\de}{K_1^{3/2}}}\Om\pa{\pf{K_1}{2\pi K_3}^{\fc{n-3}2}}\,d\om\\
&=\Om\pa{
\fc{\pi^{\fc n2}}{\Ga\pf n2} K_3^{n-1}
\fc{\de}{K_1^{3/2}K_3^{1/2}}
\pf{K_1}{2\pi K_3}^{\fc{n-3}2}
}\\
&= \Om\pa{\rc{\Ga\pf n2} K_3^{\fc n2} K_1^{\fc n2-3} 2^{-\fc n2}}\\
& = \Om\pa{\pf{2e}{n-2}^{\fc n2-1} (C_3n^{\rc 2})^{\fc n2}(C_1n^{\rc 2})^{\fc n2-3} 2^{-\fc n2}}\\
&=\Om(C_1^{\fc n2-3} C_3^{\fc n2}
n^{-\rc 2} e^{\fc n2}).
\end{align}
Note $K_2=C_2n>C_1\sqrt n + \ep = K_1+\ep$. Then $g=1$ on the support of $f$, so $(\nb f)g = \nb f$ and
\begin{align}
\int_{\R^n} \ve{\wh{(\nb f)g}(\om)}\,d\om
& =
\int_{\R^n} \ve{\wh{\nb f}(\om)}\,d\om\\
& = 
\Om(C_1^{\fc n2-3} C_3^{\fc n2}n^{-\rc 2} e^{\fc n2}).
\end{align}
Then by Lemma \ref{lem:bd-g},
\begin{align}
C_{f,2K_2B_n} &\ge 
 2K_2 \frac{\int_{\R^n} \ve{\widehat{(\nabla f)g}(\om)}\,d\om}{\int_{\R^n} |\widehat g(\om)|\,d\om}\\
 &= 2K_2\fc{\Om(C_1^{\fc n2-3} C_3^{\fc n2}n^{-\rc 2} e^{\fc n2})}
{O((5 eC_2
)^{\fc n2})
} = \Om\pa{5^{-\fc n2}C_1^{\fc n2-3} C_3^{\fc{n}2} C_2^{-\pa{\fc n2-1}}
n^{\rc 2}}.
\end{align}
\end{proof}

\subsection{$h$ is a composition of Barron functions}

In this section we proof Lemma~\ref{thm:comp}. In order to do that, let us first define the following set of functions:
\begin{df}
Define
$$
\Ga(A, C) := \set{f\colon\R^n\to \R}{\int_{\R^n} |\wh f(\om)|\,d\om \le A,
\int_{\R^n} \ve{\om}|\wh f(\om)|\,d\om \le C}
$$
\end{df}

Barron functions have many nice properties:

\begin{pr}[Properties of Barron constant]
\label{pr:barron}
\begin{enumerate}
\item (Subadditivity, \cite[\S IV.3]{Barron1993}) For any set $B$,
$$
C_{\sum_i \be_i f_i,B}\le \sum_i |\be_i| C_{f_i,B}.
$$
\item (Ridge functions,  \cite[\S IV.7]{Barron1993})
Suppose $f=h(\an{a,x})$, where $h:\R\to \R$ is a 1-dimensional function and $\ve{a}_2=1$. Then
$$
C_{f, rB_n} \le C_{h, [-r,r]}. 
$$
\item (Powers, \cite[\S IV.12]{Barron1993}) If $g:\R\to \R$, $g\in \Ga(a,c)$, then $g(x)^k\in \Ga(a^k, ka^{k-1}c)$. 
\item 
The function $f(x)=x$ has an extension $h$ agreeing with $x$ on $[-r,r]$, which satisfies $h(x) \in \Ga(O(r^{\fc 32}), O(r^{\fc 12}))$.
\end{enumerate}
\end{pr}

\begin{proof}
We show (4). Choose a bump function $b$ as in Lemma~\ref{lem:test} for $m=2$. 
Consider the extension $h(x) = x b_{(r)}(x)=xb\pf{x}{r}$ which is supported on $[-2r,2r]$. Because $b, b', b''$ are all bounded by a constant, on $[-2r,2r]$,
\begin{align}
|h(x)| &\le x\\
|h'(x)|&=|b_{(r)}(x) + xb_{(r)}'(x)|  \le  1 + O\pf{x}{r}\\
|h''(x)|&= |2b_{(r)}'(x) + b_{(r)}''(x)| \le O\pf{x}{r} + O\prc{r^2}.
\end{align}
Then by Lemma \ref{lem:l1-f}(2),
\begin{align}
\iny |\wh h(\om)|\,d\om &\le 2^{-\rc 2} \pa{\int_{-r}^r |h(x)|^2 + |h'(x)|^2\dx}^{\rc 2}\le O(r^{\fc32})\\
\iny |\om \wh h(\om)|\,d\om &\le 2^{-\rc 2} \pa{\int_{-r}^r |h'(x)|^2 + |h''(x)|^2\dx}^{\rc 2}\le O(r^{\rc2}).
\end{align}
\end{proof}

\begin{proof}[Proof of Theorem~\ref{thm:comp}]
By Proposition~\ref{pr:barron}(4) and (3),
the 1-dimensional function $y\mapsto y^2$ has an extension $k(y)$ with $k(y)\in \Ga(O(r^3),O(r^2))$. Thus,  
$C_{y^2, [-r,r]}\le r\iny \ve{\om}|\wh k(\om)|\,d\om = 
O(r^3)$. 

Because $x_i^2:\R^n\to \R$ is the composition of the
projection $x\mapsto \an{e_i,x}$ and the 1-dimensional function $y\mapsto y^2$ and , by 
(2), 
$$
C_{x_i^2, rB_n} \le C_{y^2, [-r,r]} \le O(r^3) 
$$
By (1), because $\ve{x}^2 = \sumo in x_i^2$, 
$$
C_{\ve{x}^2, rB_n} \le O(nr^3). 
$$

Now consider the function $h(y):= f_1(\sqrt y)$. We have, noting this is nonzero only for $x\in [K_1^2, (K_1+\ep)^2]$, and $f_1^{(i)}(\sqrt y) = O(K_3^{i+1})$,
\begin{align}
h'(y) &= \rc{2y^{\rc 2}}f_1(\sqrt y) + f_1'(\sqrt y)
= O\pa{\pf{K_3}{K_1} + K_3^2}
\\
h''(y) &= \rc{4y^{\fc 32}} f_1(\sqrt y) + \rc{4y} f_1'(\sqrt y) + \rc{2y^{\fc 12}}f_1''(\sqrt y)
=O\pa{\fc{K_3}{K_1^3} + \fc{K_3^2}{K_1^2} + \fc{K_3^3}{K_1}}.
\end{align}
Using $C_3<C_1$ we have $|h'|^2 + |h''|^2 = O(K_3^4)$.
Thus by Lemma~\ref{lem:1d}, 
$$
\iiy |\om \wh h(\om)|\,d\om =\pa{\int_{K_1^2}^{(K_1+\ep)^2}O\pa{K_3^4}}^{\rc 2}=
O\pa{\pa{\fc{K_1}{K_3}O(K_3^4)}^{\rc 2}}=
O\pa{K_1^{\fc 12}K_3^{\fc 32}}.
$$
Thus $f_1(\sqrt x)$ is $O(sC_1^{\fc 12}C_3^{\fc 32}n^2)$-Barron on $[-s,s]$.
\end{proof}